\documentclass{article}

\usepackage[preprint]{neurips_2024}

\usepackage[utf8]{inputenc} 
\usepackage[T1]{fontenc}    
\usepackage{hyperref}       
\usepackage{url}            
\usepackage{booktabs}       
\usepackage{amsfonts}       
\usepackage{nicefrac}       
\usepackage{microtype}      
\usepackage{xcolor}         
\usepackage{amsfonts,amssymb}
\usepackage{amsmath}
\usepackage{amsthm}
\usepackage{enumitem}

\usepackage{float}  
\usepackage{subfigure}
\newtheorem{thm}{Theorem}

\usepackage{easyReview}
\newtheorem{remark}{Remark}
\newtheorem{corollary}{Corollary}

\title{Towards Analyzing and Understanding the Limitations of DPO: A Theoretical Perspective}

\newcommand*\samethanks[1][\value{footnote}]{\footnotemark[#1]}

\author{Duanyu Feng$^{1,}$\thanks{Completed during the internship of Beijing Academy of Artificial Intelligence} \,, Bowen Qin$^{2}$, Chen Huang$^{1}$, Zheng Zhang$^{2,}$\thanks{Corresponding authors} \,, Wenqiang Lei$^{1,}$\samethanks\\
$^1$ Sichuan University\\
$^2$ Beijing Academy of Artificial Intelligence \\
\texttt{fengduanyu@stu.scu.edu.cn}, \texttt{huangc.scu@gmail.com}\\
\texttt{\{bwqin,zhangzheng\}@baai.ac.cn}, \texttt{wenqianglei@scu.edu.cn}}

\begin{document}

\maketitle

\begin{abstract}
 Direct Preference Optimization (DPO), which derives reward signals directly from pairwise preference data, has shown its effectiveness on aligning Large Language Models (LLMs) with human preferences. 
Despite its widespread use across various tasks, DPO has been criticized for its sensitivity to the SFT's effectiveness and its hindrance to the learning capacity towards human-preferred responses, leading to less satisfactory performance. To overcome those limitations, the theoretical understanding of DPO are indispensable but still lacking.
To this end, we take a step towards theoretically analyzing and understanding the limitations of DPO. Specifically, we provide an analytical framework using the field theory to analyze the optimization process of DPO.
By analyzing the gradient vector field of the DPO loss function, we find that the DPO loss function decreases the probability of producing human dispreferred data at a faster rate than it increases the probability of producing preferred data.
This provides theoretical insights for understanding the limitations of DPO discovered in the related research experiments, thereby setting the foundation for its improvement.
\end{abstract}

\section{Introduction}
Recent progress in instruction tuning~\citep{ouyang2022training,longpre2023flan} and human preference alignment~\citep{chung2022scaling} has enabled large language models (LLMs) to exhibit exceptional performance across a wide range of tasks~\citep{touvron2023llama, openai2023gpt4}. Specifically, LLMs that undergo supervised fine-tuning (SFT) across different tasks are anticipated to align with carefully curated human feedback and steer their response behavior accordingly. To achieve this, Direct Preference Optimization (DPO) has emerged as a popular and effective approach~\citep{rafailov2023direct}, which derives reward signals directly from pairwise preference data, thus bypassing the complexity of learning an additional reward model~\citep{christiano2017deep, bai2022constitutional}. 
In the context of DPO, a pairwise preference data takes the form of a triple $(x, y_w, y_l)$, comprising a specific prompt or question $x$, the human-preferred response $y_w$, and the dispreferred response $y_l$. These pairwise preference data are further used to increases the relative log probability of preferred to dispreferred responses, 
together with a Bradley-Terry preference model~\citep{bradley1952rank} based loss function.

Despite its widespread use across various tasks, the limitations of DPO are gradually coming to light, leading to less satisfactory performance as indicated by prior research~\citep{ethayarajh2023halos, xu2024contrastive}. Specifically, \textbf{DPO hinders the learning capacity of LLMs to generate human-preferred responses}, suggesting that LLMs after DPO tend to avoid producing human dispreferred responses but struggle to produce human-preferred responses, especially when training the LLM with the human-preferred response $y_w$ and the dispreferred response $y_l$ are literally similar~\citep{pal2024smaug}. 
Furthermore, \textbf{DPO has been criticized for its sensitivity to the SFT's effectiveness}~\citep{xu2024contrastive}. In other words, LLMs without proper and effective SFT typically exhibit subpar DPO performance. An empirical explanation for this is that SFT/instruction tuning are crucial for LLMs to comprehend and adhere to human directives before aligning with curated human feedback ~\citep{bai2022training}.
Despite these empirical observations, there is still a lack of theoretical analysis and understanding of the defects in DPO, which hinders insights into future directions for improving DPO.

In this paper, we take a step towards theoretically analyzing and understanding the limitations of DPO.
Focusing on the sensitivity to the SFT's effectiveness and the hindrance to the learning capacity of LLMs to generate human-preferred responses, we provide an analytical framework using the field theory~\citep{butcher2016numerical,mescheder2017numerics} to provide a comprehensive understanding of optimization process of DPO, which helps reveal the theoretical explanations behind the limitations in an unified manner. To achieve this, we begin with analyzing the gradient vector fields of DPO, which represents the direction and magnitude of the fastest decrease of the loss function of DPO over two variables: \textit{probabilities of generating human preferred and dispreferred data}, i.e., $\pi(y_w|x) \in [0,1]$ and $\pi(y_l|x) \in [0,1]$. This helps analyze how LLMs learn to steer their response behavior during the DPO optimization process. %
Remarkably, our findings suggest that \textbf{the DPO loss function decreases the probability of producing human dispreferred data at a faster rate than it increases the probability of producing preferred data}.
This provide two theoretical insights for understanding DPO's limitations: 
\begin{itemize}[leftmargin=*]
    \item \textit{Why does DPO hinder the learning capacity of LLMs to generate human-preferred responses}: In comparison to learning to generate human-preferred responses, the DPO loss function demonstrates a tendency for LLMs to readily learn to avoid generating responses that humans disprefer. This is due to the more significant impact of the DPO loss on $\pi(y_l|x)$ because of the larger gradient, as opposed to its effect on $\pi(y_w|x)$.
    \item \textit{Why is the DPO sensitive to the SFT's effectiveness}: The magnitudes and directions in various areas of the gradient vector field of DPO vary, suggesting that the practical optimization process of DPO is sensitive to the initial conditions of the alignment capability of LLMs after SFT, specifically $\pi(y_w|x)$ and $\pi(y_l|x)$. Consequently, in conjunction with the analysis of the first limitation, when SFT's effectiveness is slightly lacking, the slow increase in the probability of generating preferred responses causes the SFT-ed LLMs to struggle to align with human preferences.
\end{itemize}

In conclusion, this paper offers a theoretical analysis and comprehension of the limitations of DPO through an analytical framework employing field theory, particularly emphasizing the limitations regarding the sensitivity to the effectiveness of SFT and the impact on the ability to learn human-preferred responses. 

\section{Preliminaries}
\label{sec:relwork}
\textbf{Human Preference Alignment}. The purpose of human preference alignment is to steer the response behavior of LLMs and align their responses with human preference. Formally, given a specific question or prompt $ x $, an aligned LLM should generate human-preferred response $ y_w $ with a greater probability than human-dispreferred one $ y_l $. To achieve this, there are two technological approaches: reinforcement learning based and non-reinforcement learning based methods. As a primary focus within the reinforcement learning approach, Reinforcement Learning Human (or AI) Feedback (RLHF/RLAIF)~\citep{christiano2017deep, bai2022constitutional} aims to aims to directly evaluate and optimize responses generated by LLM. These methods initially train a reward model (RM) to evaluate human preferences, where the reward model can be iteratively trained to improve its performance ~\citep{touvron2023llama}. 
Subsequently, RLHF/RLAIF establish a reinforcement learning framework for LLMs to learn an optimal or nearly-optimal policy that maximizes the reward from the reward model using Proximal Policy Optimization (PPO)~\citep{schulman2017proximal}. While this process significantly ensures the alignment effect of LLMs, the training complexity and convergence of PPO often present practical implementation challenges ~\citep{engstrom2020implementation}.

Consequently, non-reinforcement learning based methods have been proposed. For instance, researchers have suggested simplifying the computation of PPO through the use of Direct Preference Optimization (DPO)~\citep{rafailov2023direct} and its variance such as $f$-DPO~\citep{wang2023beyond} and Kahneman-Tversky Optimization (KTO)~\citep{ethayarajh2023halos}. Notably, DPO is the first method to eliminate the training phase of the reward model and reinforcement learning, instead directly employing the LLM itself to approximate the reward model and train itself using collected paired human preference and dispreference data.

\textbf{Limitations of DPO}. Researchers have found that several limitations hinder the utilization of DPO, may experiencing negative effects after DPO~\citep{ethayarajh2023halos, xu2024contrastive}. 
\begin{itemize}[leftmargin=*]
    \item Empirical evidence suggest that the effectiveness of DPO have strong reliance on the training effect of the LLMs after SFT~\citep{bai2022training, ouyang2022training, all2024anon}. Although existing efforts have tried to solve this limitation, for example, by introducing the contrastive preference optimization (CPO)~\citep{xu2024contrastive}, curriculum learning~\citep{xu2023contrastive}, and margin-enhanced loss function~\citep{amini2024direct, pal2024smaug, qiu2024rethinking}, the reason behind this limitation still lacks theoretical explanations. 
    \item Empirical evidence also suggest that LLMs, together with DPO, struggle to learn to generate responses that aligned with human preference~\citep{azar2023general}. This is particularly true when the edit distance of $y_w$ and $y_l$ in the same pairwise preference data are close~\citep{pal2024smaug}. Furthermore, \citet{azar2023general} try to analyze the loss of DPO via the KL-regularization of the LLM before and after the modification of DPO in its hidden reward model. They find that the strength of the KL-regularisation becomes weaker and weaker the more deterministic the preferences. However, their analysis focus on explain the limitation they have discovered, making it difficult to generalize to other limitations.
\end{itemize}

Therefore, there is an urgent need for a more comprehensive theoretical analysis of DPO. On one hand, this can deepen our understanding of the role of DPO in aligning with human preferences. On the other hand, we are attempting to unify the explanation of the current limitations of DPO from a higher perspective and indicate potential directions for improvement.



\section{Understanding the Limitations of DPO}
\label{sec:theory}
Previous studies have observed that DPO has been criticized for its sensitivity to the SFT's effectiveness and hinders the learning capacity of LLMs to generate human-preferred responses. In this section, we take a step towards theoretically analyzing and understanding the limitations of DPO using field theory.

\subsection{Analyzing the Loss of DPO}
\textbf{Re-formalizing DPO Loss Function}. Given a pairwise preference data $(x, y_w, y_l) \in D$, such as HH \cite{bai2022training} and SHP \cite{pmlr-v162-ethayarajh22a}, the purpose of DPO is to make the probability of LLMs generating human preference response $ y_w $ given $ x $, denoted as $ \pi_\theta (y_w|x) $ higher than the probability of generating human dispreference response $ y_l $, denoted as $ \pi_\theta (y_l|x) $, where $ \theta $ is the parameters of LLMs. Additionally, the DPO loss function introduces $ \pi_{ ref } $, which is the probability of the reference model (usually initiated as the $ \pi_\theta $), to compare the difference between the optimized LLMs and the reference model. According to the origin paper of DPO \cite{rafailov2023direct}, its loss can be written in the following form:
\begin{equation}
\begin{split}
    \mathcal{L}_{DPO}(\pi_\theta;\pi_{ref})=&-\mathbb{E}_{(x,y_w,y_l)\sim\mathcal{D}}\left[\log \sigma\left(\beta\log\frac{\pi_\theta(y_w|x)}{\pi_{ref}(y_w|x)}-\beta\log\frac{\pi_\theta(y_l|x)}{\pi_{ref}(y_l|x)}\right)\right]\\
    =&-\left[\log \sigma\left(\beta\log\frac{\pi_\theta(y_w|x)}{\pi_{ref}(y_w|x)}-\beta\log\frac{\pi_\theta(y_l|x)}{\pi_{ref}(y_l|x)}\right)\right] \\
    =& -\log\left(\frac{x^\beta_1}{x^\beta_1+x^\beta_2}\right),
\end{split}
 \label{eq:our_dpo}
\end{equation}
where $\beta$ is a hyper-parameter\footnote{Usually, $\beta\in[0.1, 0.5].$} and $\sigma$ is the sigmoid function. For easing the calculation, we denote $ \frac{ \pi_\theta (y_w|x) }{ \pi_{ ref } (y_w|x) } = x_1 $ and $ \frac{ \pi_\theta (y_l|x) }{ \pi_{ ref } (y_l|x) } =x_2 $. In this case, to minimize the loss, we could increase $x_1$ and decrease $x_2$. 

\textbf{Gradient Vector Field of DPO}. We calculate the respective derivatives of DPO loss function (i.e., Equation (\ref{eq:our_dpo})) regarding $x_1$ and $x_2$, respectively, and construct the corresponding gradient field $\nabla \mathcal{ L }_{ DPO } (x_1, x_2)  $ to visualize the optimization behavior of DPO, revealing the dynamic features of DPO in an intuitive way. 

\begin{thm}
The partial derivatives of Equation (\ref{eq:our_dpo}) with respect to $x_1$ and $x_2$ are given by:
\begin{equation}
\left\{
    \begin{aligned}
    &\frac{\partial\mathcal{L}_{DPO}(x_1;x_2)}{\partial x_1} &=& -\frac{\beta x^\beta_2}{x_1(x^\beta_1+x^\beta_2)},\\
    &\frac{\partial\mathcal{L}_{DPO}(x_1;x_2)}{\partial x_2} &=& \frac{\beta x^{\beta-1}_2}{x^\beta_1+x^\beta_2}.
\end{aligned}
    \right.
\label{eq:g_dpo}
\end{equation}
\label{thm1}
\end{thm}

\begin{proof}
    We leave the detailed proof in Appendix \ref{ap:po1}.
\end{proof}

\begin{corollary}
For each pairwise preference data $ (x, y_w, y_l) \in D $, the update rate of $ x_1 $ in $ \mathcal { L } _ { DPO } (x_1, x_2) $ with respect to $ x_2 $, which represents the ratio of the increase in the probability of a human-preferred response to the decrease in the probability of a human-dispreferred response, is $ \frac { x_2 } { x_1 } $.
\begin{equation}
   \left| \frac{\partial\mathcal{L}_{DPO}(x_1;x_2)}{\partial x_1}/\frac{\partial\mathcal{L}_{DPO}(x_1;x_2)}{\partial x_2}\right| = \frac{ x_2 }{ x_1 }.
   \label{eq:cor1}
\end{equation}
\label{cor1}
\end{corollary}

\begin{remark}
\label{remark1}
    For any pairwise preference data, the update rate $x_2/x_1 < 1$ holds.
\end{remark}
Given that $x_1= \frac{\pi_\theta(y_w|x)}{\pi_{ref}(y_w|x)}$ and $x_2= \frac{\pi_\theta(y_l|x)}{\pi_{ref}(y_l|x)}$ are two probability ratios, where $\pi_\theta(y|x) \in [0,1]$ and $\pi_{ref}(y|x) \in [0,1]$. Assuming $\pi_{ref}(y|x)$ is the probability of the fixed reference model, we can assume $\pi_{ref}(y_w|x)=\frac{1}{a}$ and $\pi_{ref}(y_l|x)=\frac{1}{b}$, where $(a,b \geq 1)$. In this case, we have $x_1 \in [0,a]$ and $x_2 \in [0,b]$. As the DPO optimization progresses, $x_1$ tends to increase and $x_2$ tends to decrease. Consequently, $\pi_\theta(y_w|x)$ will be greater than $\frac{1}{a}$, and $\pi_\theta(y_l|x)$ will be smaller than $\frac{1}{b}$. In other words, this implies that $x_1 = \frac{\pi_\theta(y_w|x)}{\pi_{ref}(y_w|x)}$ is greater than 1, $x_2 = \frac{\pi_\theta(y_l|x)}{\pi_{ref}(y_l|x)}$ is less than 1, and therefore $x_2 < x_1$.

\subsection{Analyzing the Optimization Process of DPO}
This section aims to investigate the impact regarding generation probabilities of preference and dispreference data, we visualize the optimization plane (loss landscape) and gradient field $ \nabla \mathcal { L } _ { DPO } (x_1, x_2) =\left( \frac { \partial \mathcal { L } _ { DPO } (x_1;x_2) } { \partial x_1 } , \frac { \partial \mathcal { L } _ { DPO } (x_1;x_2) } { \partial x_2 } \right)$ in Figure \ref{fig:opt}. Since $ \pi _ { ref } (y_w|x) $ and $ \pi _ { ref } (y_l|x) $ are constants determined by the initial reference model, which may cause stretching or compression of the figure, rather than causing formal changes, we omit the denominators in $x_1$ and $x_2$\footnote{We set both $ \pi _ { ref } (y_w|x) $ and $ \pi _ { ref } (y_l|x) $ equal to 1, which simulates the situation where the reference model is absent \cite{xu2024contrastive}.}. We interpret Figure \ref{fig:opt} in various scenarios, specifically when $x_1$ is extremely large or very small, and when $x_2$ is extremely large or very small.

\begin{figure}[htbp]
  \centering  
  \subfigbottomskip=2pt 
  \subfigcapskip=-5pt 
  \subfigure[The optimization plane (loss landscape) of DPO]{
    \includegraphics[scale=0.45]{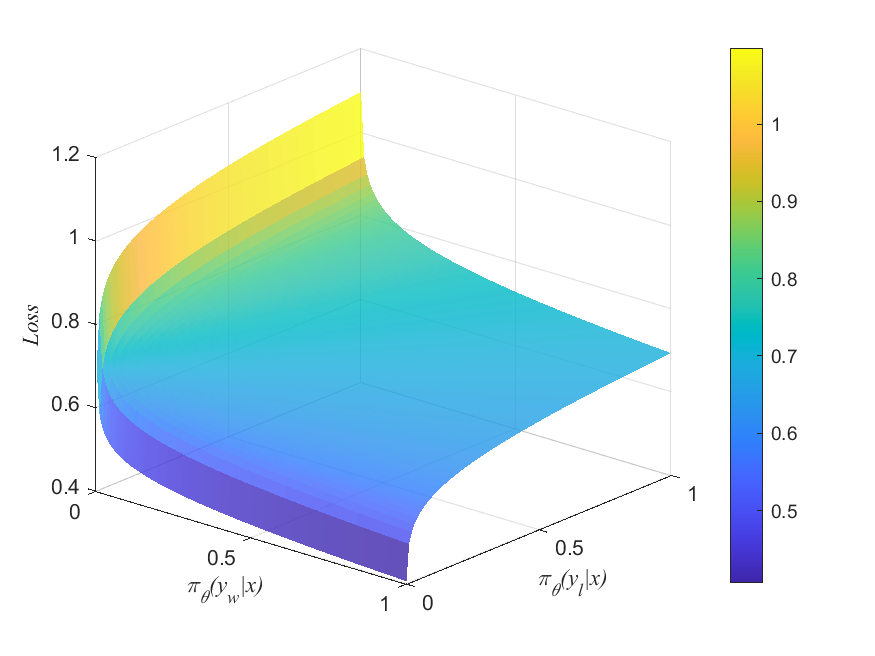}
        \label{fig:opt1}}
  \subfigure[The gradient field of DPO]{
    \includegraphics[scale=0.45]{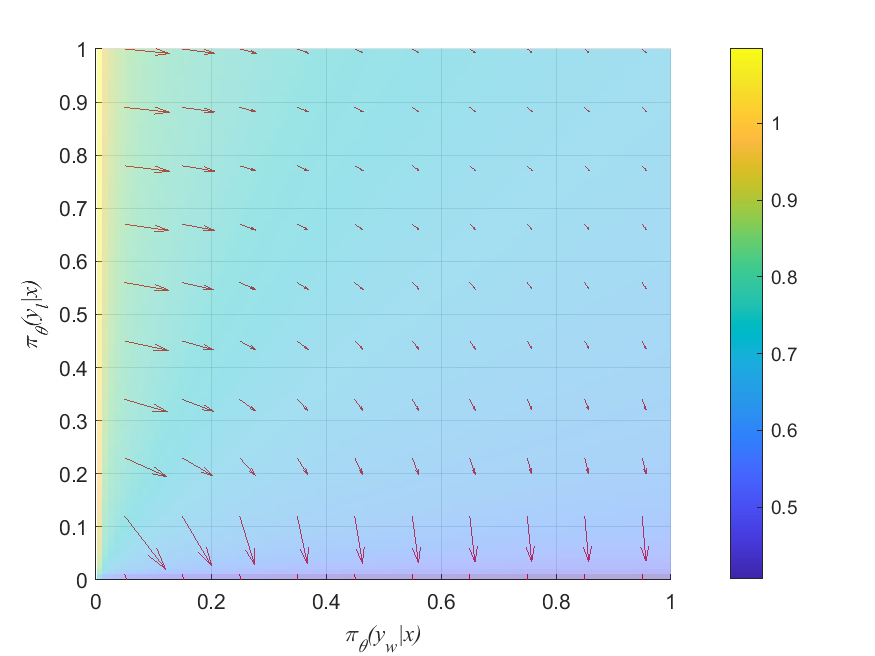}
        \label{fig:opt2}}
  
  \caption{The optimization plane (loss landscape) and gradient field of DPO. Figure (a) illustrates the values of DPO loss under different probabilities of generating prefer and disprefer responses, known as the optimization plane (loss landscape) of DPO. Figure (b) provides a top-down view of the optimization plane (loss landscape) and incorporates the gradient field at different positions using red arrows. The direction of the red arrows represents the gradient-based optimization direction, while the length of the red arrows represents magnitudes.}
 \label{fig:opt}
\end{figure}

As depicted in Figure \ref{fig:opt}, the gradient vector field vanishes in the area of low $x_1$ and then moves away, but it converges towards the region of low $x_2$ to vanish there. Consequently, the optimization objective of DPO facilitates LLM in learning how to produce responses aligning with human preferences and refraining from generating responses that do not align with human preferences. However, the magnitudes in different areas of the gradient space vary, which influences the practical optimization process of DPO. In this section, we highlight the following features  of the gradient field, which imply that DPO might be sensitive to the initial conditions of variables $x_1$ and $x_2$, which reflect the potential reliance on the alignment capability of LLMs after SFT.

\begin{itemize}[leftmargin=*]
    \item When $x_1$ is extremely small and $x_2$ is extremely large (this mainly occurs in the initial stages of optimization), as depicted in the top left corner of Figure \ref{fig:opt2}, the LLMs essentially lack the capability to produce preferred responses and tend to generate non-preferred responses. In this scenario, the gradient flow of DPO tends to rapidly increase $x_1$ while causing only minor changes to $x_2$.
    \item When both $x_1$ and $x_2$ are extremely large (this also mainly occurs in the initial stages of optimization), as illustrated in the top right corner of Figure \ref{fig:opt2}, the LLMs are capable of producing both preferred and non-preferred responses with large probabilities. In this scenario, the gradient flow of DPO tends to concurrently increase $x_1$ and decrease $x_2$, but the overall change is relatively minor, potentially resulting in difficulty escaping saddle points.
    \item When $x_2$ is extremely small (this may occur in the any stages of optimization), as depicted in the lower part of Figure \ref{fig:opt2}, indicating that the LLMs have limited capability to generate both preferred and non-preferred responses, the gradient flow of DPO tends to rapidly decrease $x_2$ while causing only minor changes to $x_1$.
\end{itemize}

\subsection{Limitation Analysis}
Expanding on our previous findings, this section seeks to offer a detailed analysis of the limitations of DPO, setting the foundation for its improvement.

\subsubsection{Limitation 1: Hindrance to the learning capacity towards human-preferred responses}
Empirical evidence indicates that LLMs, in conjunction with DPO, encounter challenges in learning to produce responses that align with human preference. In the following section, our theoretical findings further support this empirical evidence.

\begin{remark}
\label{remark2}
    Compared to learning to generate human-preferred responses, the DPO loss function shows a tendency for LLMs to easily learn how to avoid generating responses that humans disprefer due to the more significant impact of the DPO loss on $x_2$.
\end{remark}

According to our Remark \ref{remark1}, the impact of the DPO loss on $x_2$ is more significant due to the larger gradient $\frac{\partial\mathcal{L}_{DPO}(x_1;x_2)}{\partial x_2}$ compared to the impact on $x_1$, which has a smaller gradient $\frac{\partial\mathcal{L}_{DPO}(x_1;x_2)}{\partial x_1}$. As $x_1$ tends to increase and $x_2$ tends to decrease during optimization, we have $ \frac { x_2 } { x_1 } \rightarrow 0 $. At this point, DPO focuses more on updating $x_2$ to approach 0, while making minimal updates to $x_1$ (due to larger gradient $\frac{\partial\mathcal{L}_{DPO}(x_1;x_2)}{\partial x_2}$). In other words, DPO concentrates excessively on indicating to LLMs what constitutes a poor response, while neglecting to guide LLMs on what constitutes a good response that aligns with human preference. Informally, in extreme scenarios, if the human-preferred response $y_w$ and the dispreferred response $y_l$ are literally similar, the gradient with respect to $x_2$ would counteract the gradient of $x_1$ to some extent, thereby weakening the optimization toward $x_1$ and leading to the hindrance to the learning capacity towards human-preferred responses.

\subsubsection{Limitation 2: Sensitivity to SFT's effectiveness}
While SFT has become one of crucial techniques for aligning LLMs with human language, LLMs following SFT may demonstrate differing levels of alignment as a result of factors such as data quality and training strategies. As evidenced by previous studies, the effectiveness of DPO is dependent on the alignment capability of LLMs following SFT, and subpar SFT may result in a reduction of LLM effectiveness after DPO \cite{xu2024contrastive, ethayarajh2023halos}. In the following section, we offer theoretical explanations for this limitation. To start, we uncover characteristics when handling LLMs with various initial positions within the gradient field of DPO.

\begin{itemize}[leftmargin=*]
    \item When the initial position of LLMs is situated at the lower right corner of Figure \ref{fig:opt2}, the focus of DPO shifts to reducing the probability of generating dispreferred responses. In this situation, DPO demonstrates its ability to rapidly diminish this probability and prevent LLMs from generating responses that are dispreferred by humans.
    \item When the initial position of LLMs is situated at the left side of Figure \ref{fig:opt2}, DPO's focus shifts to enhancing the probability of generating human-preferred responses. However, DPO may not be able to swiftly increase this probability, as the gradient direction favors optimizing $x_2$.
\end{itemize}

\begin{remark}
    The alignment capability of SFT-ed LLMs may significantly impact DPO, leading to bad initial states for LLM in the optimization plane (loss landscape) of DPO.
\end{remark}
The initial states in the gradient vector field have a significant impact on the final optimization results. As depicted in Figure \ref{fig:opt}, the optimization plane (loss landscape) and gradient field of DPO in different regions can drive LLM to different optimization results, potentially leading to instability. 
In such instances, LLMs that have not undergone satisfactory SFT often exhibit limited proficiency in effectively adhering to instructions and responding to human queries. 
The initial positioning of these SFT-ed LLMs may be situated in the lower-left corner of Figure \ref{fig:opt2}, indicating low probabilities for generating both preferred and dispreferred responses, and a gradient direction that does not entirely prioritize the enhancement of human-preferred response probabilities. Alternatively, the initial positioning of these SFT-ed LLMs may be in the upper-right corner of Figure \ref{fig:opt2}. In this scenario, the presence of very small gradients in the upper-right corner can lead to sluggish convergence and challenges in escaping local minima. Consequently, this can result in inadequate learning of human preference data.

\section{Conclusion}
\label{sec:con}

In this paper, we focus on offering a theoretical analysis and comprehension of the limitations of DPO through an analytical framework employing field theory. 
By analyzing the gradient vector fields of DPO, we find that the DPO loss function decreases the probability of producing human dispreferred data at a faster rate than it increases the probability of producing preferred data. 
This finding can be explained from a unified perspective of DPO regarding the sensitivity to the effectiveness of SFT and the hindrance to the learning capacity of LLMs in generating human-preferred responses.
In the future, we will conduct experiments to validate our theory and make improvements to DPO based on our finding.

\bibliographystyle{plainnat}
\bibliography{references}


\appendix
\section{The Proof of Theorem \ref{thm1}}
\label{ap:po1}

\begin{thm}
The partial derivatives of $\mathcal{L}_{DPO}(x_1,x_2)=-\log(\frac{x^\beta_1}{x^\beta_1+x^\beta_2})$ with respect to $x_1$ and $x_2$ are given by:
\begin{equation}
\left\{
    \begin{aligned}
    &\frac{\partial\mathcal{L}_{DPO}(x_1;x_2)}{\partial x_1} &=& -\frac{\beta x^\beta_2}{x_1(x^\beta_1+x^\beta_2)},\\
    &\frac{\partial\mathcal{L}_{DPO}(x_1;x_2)}{\partial x_2} &=& \frac{\beta x^{\beta-1}_2}{x^\beta_1+x^\beta_2}.
\end{aligned}
    \right.
\end{equation}
\end{thm}

\begin{proof}

    For $\frac{\partial\mathcal{L}_{DPO}(x_1;x_2)}{\partial x_1}$,
    \begin{equation}
    \begin{aligned}
    &\frac{\partial\mathcal{L}_{DPO}(x_1;x_2)}{\partial x_1} &=& - \frac{x^\beta_1 +x^\beta_2}{x^\beta_1} (\frac{\beta x^{\beta - 1}_1}{x^\beta_1 +x^\beta_2}+\frac{-\beta x^{2\beta-1}_1}{(x^\beta_1+x^\beta_2)^2})\\
    & &=& - \frac{\beta x^{\beta-1}_1(x^\beta_1+x^\beta_2)-\beta x^{2\beta-1}_1}{x^\beta_1(x^\beta_1+x^\beta_2)}\\
    & &=& - \frac{\beta x_1^{\beta-1} x^\beta_2}{x_1^\beta(x^\beta_1+x^\beta_2)}\\
    & &=& - \frac{\beta x^\beta_2}{x_1(x^\beta_1+x^\beta_2)}.\\
    \end{aligned}
    \end{equation}

    For $\frac{\partial\mathcal{L}_{DPO}(x_1;x_2)}{\partial x_2}$,
    \begin{equation}
    \begin{aligned}
    &\frac{\partial\mathcal{L}_{DPO}(x_1;x_2)}{\partial x_2} &=& \frac{1}{x^\beta_1 +x^\beta_2} \beta x^{\beta-1}_2\\
    & &=& \frac{\beta x^{\beta-1}_2}{x^\beta_1+x^\beta_2}.\\
    \end{aligned}
    \end{equation}

\end{proof}

\end{document}